\documentclass{article}

\usepackage[preprint]{nips_2018}

\usepackage[utf8]{inputenc} 
\usepackage[T1]{fontenc}    
\usepackage{url}            
\usepackage{booktabs}       
\usepackage{amsfonts}       
\usepackage{nicefrac}       
\usepackage{microtype}      

\usepackage{color}
\usepackage{amssymb}
\usepackage{amsmath}
\usepackage{amsthm}
\usepackage{textcomp}
\usepackage{graphicx}

\usepackage{stmaryrd} 
\usepackage{enumitem}
\usepackage{algorithm2e}

\usepackage{xspace}
\makeatletter
\DeclareRobustCommand\onedot{\futurelet\@let@token\@onedot}
\def\@onedot{\ifx\@let@token.\else.\null\fi\xspace}
\def\iid{{i.i.d}\onedot}
\def\eg{{e.g}\onedot} 
\def\ie{{i.e}\onedot}

\makeatother

\newtheorem{lemma}{Lemma}
\newtheorem{theorem}{Theorem}
\newtheorem{assumption}{Assumption}

\newtheorem{definition}{Definition}

\newcommand{\nR}{\mathbb{R}}

\newcommand{\nC}{\mathcal{C}}
\newcommand{\nD}{\mathcal{D}}

\newcommand{\nF}{\mathcal{F}}

\newcommand{\nH}{\mathcal{H}}

\newcommand{\nL}{\mathcal{L}}

\newcommand{\nN}{\mathcal{N}}

\newcommand{\nX}{\mathcal{X}}
\newcommand{\nY}{\mathcal{Y}}
\newcommand{\nZ}{\mathcal{Z}}

\newcommand{\ra}{\rightarrow}

\newcommand{\PP}[1]{\mathbb{P}\left[ #1 \right] }
\newcommand{\PPp}[1]{\mathbb{P}[ #1 ] }

\newcommand{\PPc}[2]{\mathbb{P}\left[\left.#1\right| #2 \right] }

\newcommand{\EEc}[2]{\mathbb{E}\left[\left.#1\right| #2 \right] }
\newcommand{\EEd}[2]{\mathbb{E}_{#2}\left[ #1 \right] }

\newcommand{\Ind}[1]{\mathbb{I}\left[ #1 \right] }

\newcommand{\abs}[1]{\left| #1 \right|}
\newcommand{\absd}[1]{| #1 |}
\newcommand{\supp}[1]{\operatornamewithlimits{supp}(#1)}

\newcommand{\by}{\mathbf{y}}

\newcommand{\bz}{\mathbf{z}}

\newcommand{\argmin}{\operatornamewithlimits{argmin}}

\newcommand{\riska}[2]{R_{#2}( #1 )}

\newcommand{\mrisk}[1]{R_{mar}(h)}
\newcommand{\hmrisk}[1]{\hat{R}_{mar}(h)}

\newcommand{\rot}[1]{\rotatebox[origin=c]{90}{#1}}

\usepackage[small,compact]{titlesec}

\titlespacing\section{0pt}{4pt plus 2pt minus 2pt}{0pt plus 2pt minus 2pt}
\titlespacing\subsection{0pt}{4pt plus 2pt minus 2pt}{0pt plus 2pt minus 2pt}

\title{MACRO: A Meta-Algorithm for Conditional Risk Minimization} 

\author{
  Alexander~Zimin \\
  IST Austria\\
  \texttt{azimin@ist.ac.at} \\
   \And
   Christoph H.~Lampert \\
   IST Austria \\
   \texttt{chl@ist.ac.at} \\
}

\begin{document}

\maketitle

\begin{abstract}
We study \emph{conditional risk minimization (CRM)}, \ie 
the problem of learning a hypothesis of minimal risk for 
prediction at the next step of 
sequentially arriving dependent data.
Despite it being a fundamental problem, successful learning 
in the CRM sense has so far only been demonstrated 
using theoretical algorithms that cannot be used for 
real problems as they would require storing all incoming data. 
In this work, we introduce MACRO, a meta-algorithm for CRM 
that does not suffer from this shortcoming, but nevertheless
offers learning guarantees. Instead of storing all data it 
maintains and iteratively updates a set of learning 
subroutines. 
With suitable approximations, MACRO
applied to real data, yielding improved prediction 
performance compared to traditional non-conditional learning.
\end{abstract}

\section{Introduction}
Conditional risk minimization (CRM) is a fundamental learning problem 
when the available data is not an \iid sample from a fixed data 
distribution, but a sequence of interdependent observations, \ie 
a \emph{stochastic process}. 
Like \iid samples, stochastic processes can be interpreted in a 
generative way: each data point is sampled from a \emph{conditional 
data distribution}, where the conditioning is on the sequence of 
observations so far. 
This view is common, \eg, in the literature on time-series prediction,
where the goal is to predict the next value of a time-series given the
sequence of observed values to far. 
CRM is the discriminative analog to this: for a given loss function
and a set of hypotheses, the goal is to identify the best hypothesis 
to apply at the next time step.

Conditional risk minimization has many application for learning tasks in which data arrives sequentially and decisions have to be made quickly, e.g. frame-wise classification of video 
streams.
As a more elaborate example, imagine the problem of predicting which 
flights will be delayed over the next hour at an airport.
Clearly, the observed data for this task has a temporal 
structure and exhibits strong dependencies.
A prediction function that is adjusted to the current conditions of the airport will be very useful. It will not only make it possible to predict the delays of incoming flights in advance, but also to identify the possible parameters changes that would avoid delays, \eg by rescheduling or rerouting aircrafts \citep{Rosenberger2003}.

Despite its fundamental nature, CRM remains largely an unsolved 
problem in theory as well as in practice. One reason is that it 
is a far harder task than ordinary \iid learning, because the 
underlying conditional distributions, and therefore the 
optimal hypotheses, change at every time step. 
Until recently, it was not even clear how to formalize 
\emph{successful learning} in the CRM context, as there exists 
no single target hypothesis to which a sequential learning 
process could converge.
\cite{zimin2016aistats} was the first work to offer a handle on 
the problem by formalizing a notion of \emph{learnability} in 
the CRM sense, see Section~\ref{sec:theory} for a formal discussion.
Unfortunately, their results are purely theoretical as the suggested 
procedure would require storing all observations of the stochastic 
process for all future time steps, which is clearly not possible 
in practice. 

In this work, we make three contributions. 
First, we generalize the notion of learnability from~\cite{zimin2016aistats} to a more practical 
approximate \emph{$\epsilon$-learnability}, resembling the 
classic \emph{probably approximately correct (PAC)} framework~\citep{valiant1984theory}.
Second, we introduce MACRO, a meta-algorithm for CRM, and prove that it achieves $\varepsilon$-learnability under less restrictive assumptions than previous approaches.
At last, we show that MACRO is \emph{practical}, as it requires only training a set of elementary learning algorithms on different subsets of the available data, but not storing the 
complete sequence of observations, and report on practical experiments that highlight MACRO's 
straight-forward applicability to real-world problems.

\paragraph{Related work}\label{sec:relatedwork}
The fundamentals of statistical learning theory were build originally based on the assumption of independent and identically distributed data \citep{Vapnik01},
but very soon extensions to stochastic processes were suggested.
%

When one is more interested in short-term behaviour of the process, it makes sense to focus on the \emph{conditional risk}, where the expectation is taken with respect to the conditional distribution of the process, as it was argued in~\citep{Pestov2010,Shalizi13}.
Learnability, \ie the ability to perform a risk minimization with respect to the conditional risk, was established for a number of particular classes of stochastic processes, such as \iid, exchangeable, mixing and some others, see \citep{steinwart2005consistency,Pestov2010,Berti04,Mohri02fixed,zimin2015arxiv}.
Most of these works focus on the estimation of the conditional risk by the average of losses over the observed data so far. 
Conditional risk minimization was considered by \citep{kuznetsov2015learning} with a later extension in \citep{kuznetsov2016time}.
Without trying to achieve learnability, they consider the behaviour of the empirical risk minimization algorithm at each fixed time step by using a non-adaptive estimator.
A general study of learnability was performed in \citep{zimin2016aistats} by emphasizing the role of pairwise discrepancies and the necessity to use an adaptive estimator. 
A number of related setting that utilize the notion of conditional risk and its variants have been studied in \cite{Kuznetsov01,zimin2015arxiv,wintenberger2014optimal}.

%
%

A related topic to conditional risk minimization is time series prediction.
While, in general, time series methods can not be applied to CRM, both fields share a number of ideas.
Traditional approaches to prediction include forecasting by fitting different 
parametric models to the data, such as ARMA or ARIMA, or using 
spectral methods, see, \eg,~\citep{box2015time}. 
Alternative approaches include nonparametric prediction of 
time series~\citep{Modha01,modha1998memory,alquier2012model}, 
and prediction by statistical learning~\citep{alquier2013prediction,mcdonald2012time}.
Similar research problems are studied in the field of dynamical systems, when one tries to identify the underlying transformation that 
governs the transitions, see \citep{nobel2001consistent,farmer1987predicting,casdagli1989nonlinear,steinwart2009consistency}.
%

%
%


\section{Conditional Risk Minimization}\label{sec:theory}
%
%
We first introduce our main notations.
We are given a sequence of observations $\left\lbrace \bz_t \right\rbrace_{t=1}^{n} $ from a stochastic process taking values in some space $\nZ$.
The most common option for $\nZ$ would be a product space, $\nX\times\nY$, where $\nX$ and $\nY$ are input and output sets, respectively, of a supervised prediction problem. 
Other choices are possible, though, \eg for modeling unsupervised learning tasks. 
We write $\bz_{i:j}$ as a shorthand for a sequence $(\bz_i, \dots, \bz_j)$ for $i \leq j$.
We fix a hypotheses class $\nH$, which is usually a subset of $\left\lbrace h : \nZ \ra \nD \right\rbrace$ with $\nD$ being some decision space, \eg $\nR^d$.
We also fix a loss function {$\ell: \nD \times \nZ \ra [0,1]$, that allows us to evaluate the loss of a given hypothesis $h$ on a given point $z$ as $\ell(h(z), z) = \ell(h, z)$, with the latter version used to shorten the notations.
%
%
For any time step $t$, we denote by $\EEd{f(\bz_{t+1})}{t}$ the expectation of a function $f$ with respect to the distribution of the next step, conditioned on the data so far, \ie $\EEc{f(\bz_{t+1})}{\bz_{1:t}}$
To characterize the capacity of the hypothesis class we use sequential covering numbers with respect to the $\ell_\infty$-norm, $\nN_\infty$, introduced by \citep{Rakhlin01}.
This complexity measure is applied to the induced functions space $\nL(\nH) = \left\lbrace \ell(h, \cdot), \forall h\in\nH \right\rbrace$.
Throughout the paper we assume that $\nL(\nH)$ has a finite sequential fat-shattering dimension, a notion of complexity of a function class.
These quantities are generalizations of more traditional measures for \iid data, so readers unfamiliar with the theory of stochastic processes can read the results in terms of usual covering numbers and VC dimension.
%
The formal definitions can be found in the supplementary material. 

Our task is to find a hypothesis of minimal risk, \ie solve
\begin{equation}
\min_{h\in\nH}R_n(h),
\end{equation}
where 
$
R_n(h) = \EEd{\ell(h, \bz_{n+1})}{n}
$
is the conditional risk at step $n$.
%
Note that the conditional distribution is different at every time step, 
so the objective also constantly changes. 
At each step we can compute only an approximate solution, $h_n$, based on the data.
A desired property is an improvement of the quality with the amount of the observations, as summarized in the following definition.

\begin{definition}[\textbf{Learnability} ~\citep{zimin2016aistats}]
For a fixed loss function $\ell$ and a hypotheses class $\nH$, 
we call a class of processes $\nC$ \emph{conditionally learnable in the limit}  
if there exists an algorithm that, for every process $P$ in $\nC$, produces a sequence of hypotheses, $h_n$, each based on $\bz_{1:n}$, satisfying
\begin{equation}
\riska{h_n}{n} - \inf_{h\in\nH} \riska{h}{n} \ra 0
\label{eq:learnability}
\end{equation}
in probability over the samples drawn from $P$.
An algorithm that satisfies \eqref{eq:learnability} we call a \emph{limit learner} for the class $\nC$.
%
\end{definition}
From a practical perspective, one might be satisfied with achieving 
only some target accuracy.
For this purpose, we introduce the following relaxed definition.
\begin{definition}[\textbf{$\varepsilon$-Learnability}]
For a fixed loss function $\ell$ and a hypotheses class $\nH$, 
we call a class of processes $\nC$ \emph{$\varepsilon$-conditionally learnable} for $\varepsilon > 0$  
if there exists an algorithm that, for every process $P$ in $\nC$, produces a sequence of hypotheses, $h_n$, each based on $\bz_{1:n}$, satisfying
\begin{equation}
\PP{\riska{h_n}{n} - \inf_{h\in\nH} \riska{h}{n} > \varepsilon} \ra 0.
\label{eq:epslearnability}
\end{equation}
An algorithm that satisfies \eqref{eq:epslearnability} we call an \emph{$\varepsilon$-learner} for the class $\nC$.
%
\end{definition}

A class of processes is \emph{learnable} in the sense of~\citep{zimin2016aistats} if and only if it is 
\emph{$\varepsilon$-learnable} for all $\varepsilon>0$. 

\subsection{Discrepancies}
Following~\citep{zimin2016aistats}, our approach relies on a specific notion 
of distance between distributions called \emph{discrepancy}, 
a subclass of integral probability metrics, \citep{zolotarev1983probability}.
It is a popular measure used, for example, in the field of domain adaptation \citep{kifer2004detecting,BenDavid2007,ben2010theory}.
In this work we use this distance only to quantify distances between the conditional distributions, therefore we define the discrepancies only between those.
\begin{definition}[\textbf{Pairwise discrepancy}]
For a sample $\bz_1, \bz_2, \dots$ from a fixed stochastic process, the pairwise discrepancy between time points $i$ and $j$ is
\begin{equation}
d_{i,j} = \sup_{h\in\nH} \abs{ \riska{h}{i-1} - \riska{h}{j-1} }.
\end{equation}
\end{definition}

Given a sequence of conditional distributions, $P_t = \PPc{\cdot}{\bz_{1:t}}$, an interesting quantity is the covering number of the space $\lbrace P_0, \dots, P_n \rbrace$ with the discrepancy as a metric, which we denote by $\nN(d, n, \varepsilon)$.
If, as a thought experiments, we identify $\varepsilon$-close distributions with each 
other, then $\nN(d, n, \varepsilon)$ characterizes the minimal number of 
distributions sufficient to represent all of them.
This quantity can therefore serve as a measure of the complexity of learning
from this sequence and it appears naturally in the analysis of the performance of our algorithm as a lower bound on the necessary computational resources.

\section{A Meta-Algorithm for Conditional Risk Optimization (MACRO)}
%
%
This section contains our main contributions. 
We introduce a meta-algorithm, MACRO, that can utilize any 
ordinary learning algorithm as a subroutine, and we study 
its theoretical properties, in particular establishing its 
ability to $\varepsilon$-learn a broad class of stochastic
processes. 
For the sake of readability, for the main theorems we 
only provide their statements, while the proofs 
can be found in the supplemental material.

The main feature of the analysis provided in \cite{zimin2016aistats} is that all properties of the process under consideration are summarized in a single assumption: the existence of a specific upper bound on the discrepancies.
In the present paper we continue in this framework, but use the following weaker 
assumption.
\begin{assumption}\label{assumption}
For any $i,j$, there exist a value $M_{i,j}$ that is a function of $z_{1:\max{(i,j)}}$, such that
$d_{i,j} \leq M_{i,j}$.

\end{assumption}
In words, we must be able to upper-bound the discrepancies of the process by a quantity we can actually compute.
%
Note that the assumption is trivially fulfilled for many simple processes. For example, $M_{i,j}\equiv 0$ works for \iid data and $M_{i,j}=\Ind{\bz_{i-1} \neq \bz_{j-1}}$
for discrete Markov chains.
In contrast to \citep{zimin2016aistats} we do not require any additional measurability conditions, since those are taken care of by the way MACRO uses the upper bounds.
Further examples of upper bounds can be found in \cite{zimin2016aistats}.
Note that for real data, it might not be possible to verify 
assumption~\ref{assumption} 
with non-trivial $M_{i,j}$, but this is the case with the most assumptions made about the stochastic process in the literature, and even the \iid property in standard learning.

In principle MACRO can use any computable upper bounds, even the trivial constant 1. However, learnability is only guaranteed if they satisfy certain conditions, which we will discuss after the corresponding theorems.

\subsection{Conditional Risk Minimization with Bounded Discrepancies}
%
The main idea behind MACRO is the thought experiment from the 
previous section: if two conditional distributions are very similar, 
we can use the same hypothesis for both them. 
To find these hypotheses, the meta-algorithm maintains a list of 
learning subroutines, where each of them is run independently and 
updated using a selected subset of the observed data points.
%
%
%
Over the course of the algorithms, the meta-algorithm always 
maintains an \emph{active hypothesis} that can immediately 
be applied when a new observation arrives. 
After each observation, one or more of the existing subroutines 
are updated, and a new subroutine can be added to the list, if necessary.
The meta-algorithm then constructs a new active hypothesis 
from the ones produced by the currently running subroutines, to
be prepared for the next step of the process.
The schema of the algorithm is given in Figure \ref{fig:macro}.


%
\begin{figure}
\medskip\hrule\hrule\hrule

\smallskip
\textbf{Initialization:} $T\leftarrow \emptyset$, $N\leftarrow 0$\smallskip
\hrule\smallskip
%

\textbf{At any time point $t=1,2,\dots$:} 
\begin{itemize}[topsep=0pt,itemsep=0pt,partopsep=0ex,parsep=0ex]
\item \textbf{choose} the active hypothesis from the closest $\varepsilon$-close subroutine or a newly started one

-- identify all $\varepsilon$-close subroutines: $J = \{ 1 \leq j \leq N: M_{t,\tau_j} \leq \varepsilon\}.$\\
%
%
-- if $J=\emptyset$: create a new subroutine, $S_{N+1}$, and set $J=\{N+1\}; \tau_{N+1} \leftarrow t;  N\leftarrow N+1.$\\
%
-- set the active hypothesis: $h_{a} \leftarrow \textsf{output}(S_{j}) \text{ for } j=\argmin_{j\in J} M_{t,\tau_j}.$
\item \textbf{output} the currently active hypothesis, $h_t \leftarrow h_a$
\item \textbf{observe} the next value of the process, $\bz_t$
\item \textbf{update} all $\varepsilon$-close subroutines: $S_{j}\leftarrow \textsf{update}(S_{j},\bz_t) \text{ for all } j\in J$
%
%
%
%
\end{itemize}
\smallskip\hrule\hrule\hrule
\caption{MACRO algorithm}
\label{fig:macro}
\end{figure}

Before we proceed to the theoretical properties of the meta-algorithm, we fix further notations for its components.
At any time step $n$, we denote by $N_n$ the number of started subroutines (\ie the current value of $N$). 
The time steps in which the $j$-th subroutine is updated up to step $n$ form a set 
$C_{j,n} = \lbrace t_{j,1}, \dots, t_{j,s_{j,n}} \rbrace$ of size $s_{j,n}$. 
By $h_{j,i}$ we denote the output of the $j$-th subroutine after having been updated $i$-times.
By $I_n\in [N_n]$ we denote the index of the subroutine that MACRO outputs in step $n$,
\ie $h_n = h_{I_n, s_{I_n,n}}$. 

%
%

%
%
%

%
\paragraph{Computational considerations.}
The amount of computations done by MACRO in step $n$ is at most proportional to the current number of subroutines, $N_n$. Therefore, we first discuss the quantitative behavior of this number.
\begin{lemma}\label{lemma:computations}
Let $\nN(M,n,\varepsilon)$ be an $\varepsilon$-covering number of $\lbrace P_0, \dots, P_n \rbrace$ with respect to $M_{i,j}$'s.
Then for any $n=1,2,\dots$, it holds that 
\begin{equation}
 \nN(M, n, \varepsilon) \leq N_n \leq \nN(M, n, \varepsilon/2).
\end{equation}
\end{lemma}
%
%

Observe that $\nN(M, n, \varepsilon)$ is always lower-bounded by $\nN(d, n, \varepsilon)$, making it a natural limit on how many separate subroutines are required to learn a particular sequence.
%

The overall computational complexity of the previous algorithms based on the ERM principle from \cite{zimin2015arxiv,zimin2016aistats} is proportional to $n^2$ for a dataset of size $n$, while MACRO is able to reduce it to $n N_n$ with a potential for further reduction, which allows its application to much larger datasets as shown in Section \ref{sec:practice}.
%
%
%
\paragraph{Exceptional sets.}
As discussed in \cite{zimin2016aistats} (and resembling 
the "probably" aspect of PAC learning), learnability 
guarantees for stochastic processes may not hold for 
every possible realization of the process. 
Henceforth, we follow the same strategy and introduce a set of exceptional realizations.
However, the definition differs from the one in \citep{zimin2016aistats}, as it is adapted to the working mechanisms of the meta-algorithm.
\begin{definition}[Exceptional set]
For a fixed $n$, for any $k \geq 1$ and $1 \leq m \leq n$, set
\begin{equation}
E_{k,m} = \lbrace \abs{\supp{I_n}} \leq k \wedge \min_{j \in \supp{I_n}}s_{j,n} \geq m \rbrace,
\end{equation}
where $\supp{I_n}$ denotes the support of $I_n$.
Then $E^c_{k,m}$, the complement of $E_{k,m}$, is an exceptional set of realizations.
\end{definition}
In words, the favorable realizations are the ones that do not force 
the algorithm to use too many subroutines (at most $k$) and, at same 
time, all used subroutines are updated often enough (at least $m$ times). 
The intuition behind this is that a subroutine will be slow in converging to an optimal predictor if it is updated very rarely. 
However, the overall performance of the meta-algorithm can suffer only if rarely updated subroutines are nevertheless used from time to time. 
%

\subsection{Subroutines}\label{sec:subroutines}

MACRO, as a meta algorithm, relies on the subroutines to 
perform the actual learning of hypotheses.  
In the following sections we will go through several option for subroutines and discuss the resulting theoretical guarantees.

\paragraph{Empirical risk minimization.}
We start with the simplest choice of a subroutine: an \emph{empirical risk minimization (ERM) algorithm} 
that stores all data points it is updated with. When required, it 
outputs the hypothesis that minimizes the average loss over this 
training set. 
Formally, the $j$-th ERM subroutine outputs
\begin{equation}
h_{j,i} = \argmin_{h\in\nH} \hat{R}_n(h, j) \quad \text{for} \quad \hat{R}_n(h,j) = \frac{1}{s_{j,n}}\sum_{t\in C_{j,n}}\ell(h, \bz_t).
\end{equation}
%
%
Consequently, MACRO's output is $h_n = \argmin_{h\in\nH} \hat{R}_n(h, I_n)$
for which we can prove the following theorem.
\begin{theorem}\label{theorem:meta-erm}
If MACRO is run with ERM as a subroutine, then we have for any $k,m \geq 1, \alpha \in [0, 1]$ and $\beta \in [0, \alpha/4]$
\begin{equation}
\PP{R_n(h_n) - \inf_{h\in\nH}R_n(h) > \alpha + 2\varepsilon} \leq \frac{2k\nN_\infty(\nL(\nH),\beta, n)}{(\alpha-4\beta)^2} e^{-\frac{1}{2}m(\alpha-4\beta)^2} + \PP{E_{k,m}^c}. \notag
\end{equation}
\end{theorem}
From this theorem we can read off the conditions for learnability of the meta-algorithm.
If there exist sequences $k_n, m_n$, satisfying $\frac{m_n}{\log n} \ra \infty$ and $\PPp{E^c_{k_n,m_n}} \ra 0$, then the meta-algorithm with ERM as a subroutine is an $\varepsilon$-learner (up to a constant).
The condition on the rate of growth of $m_n$ comes from the fact that it needs to compensate for the growth of covering numbers, which is a polynomial of $n$ (see the supplementary material for more details).
The existence of such sequences $k_n$ and $m_n$ depends purely on the properties of the process (or class of processes) that the data is sampled from.
Importantly, neither $k_n$ nor $m_n$ are needed to be known by MACRO as it automatically adapts to unfavorable conditions and exploits the favorable ones.

Note that the computation of $h_n$ can be seen as a minimization of non-uniformly weighted average over the observed data, an approach proposed by \citep{zimin2016aistats}.
However, our method differs in the way the weights are computed, how the exceptional set is defined and relies on a less restrictive assumption.
%

\paragraph{Online learning.}
ERM as a subroutine is interesting from a theoretical perspective, but it defeats the 
main purpose of the meta-algorithm, namely that not all data of the process has to be
stored.  
Instead, one would prefer to rely on a subroutine that can be trained 
incrementally, \ie one sample at a time, as it is typical in \emph{online learning}.

In the following, by an \emph{online subroutine} we understand any algorithm that 
is designed to control the \emph{regret} over each particular realization, 
see \citep{Cesa-Bianchi01} for a thorough study of the problem.
The regret of the $j$-th subroutine at the step $n$ is defined as
\begin{equation}
W_{j,n} = \sum_{i=1}^{s_{j,n}} \ell(h_{j,i-1}, \bz_{t_{j,i}}) - \inf_{h\in\nH} \sum_{i=1}^{s_{j,n}} \ell(h, \bz_{t_{j,i}}).
\end{equation}
The choice of a particular subroutine depends on the loss function and the hypotheses class.
To abstract from concrete bounds and subroutines, we prove a theorem that bounds the performance of the meta-algorithm in terms of the regrets of the subroutines.
Thereby, we obtain that any regret minimizing algorithm will be efficient as a subroutine for MACRO as well.

%

As our goal is not to minimize regret, but the conditional risk, 
we perform an \emph{online-to-batch conversion} to choose the 
output hypothesis of each subroutine. 
In this work we consider two of the many existing online-to-batch conversion methods, 
one specifically for the convex losses and the other one for the general case.
\paragraph{Convex losses.}
For a convex loss function, the output of a subroutine is the average over the hypotheses it produced so far. 
In this case, MACRO's output is $h_{n} = \frac{1}{s_{I_n,n}}\sum_{i=1}^{s_{I_n,n}} h_{I_n,i}$.
%
and we can prove the following theorem. 
\begin{theorem}\label{theorem:online-convex}
For a convex loss $\ell$, if the subroutines of MACRO use an averaging for online-to-batch conversion, we have for any $\alpha \in [0,1]$ and $\beta \in [0, \alpha/8]$
\begin{equation}
\PP{R_n(h_n)\!-\!\inf_{h\in\nH}R_n(h) \!>\! \alpha \!+\! W_{I_n,n}/s_{I_n,n} \!+\! 4\varepsilon}\!\leq\ \!\! \frac{4k\nN_\infty(\nL(\nH),\beta, n)}{(\alpha/2-4\beta)^2} e^{-\frac{1}{2}m(\alpha/2-4\beta)^2}\!+\! \PP{E_{k,m}^c} 
\end{equation}
\end{theorem}
For Hannan-consistent online algorithms, $W_{I_n,n}/s_{I_n,n}$ vanishes as $s_{I_n,n}$ grows.
Hence, the same conditions as the ones given after Theorem \ref{theorem:meta-erm} ensures that MACRO is an $\varepsilon$-learner in this case.
\paragraph{Non-convex losses.}
For non-convex losses, a simple averaging for online-to-batch conversion does not work, 
so we need to perform a more elaborate procedure.
We use a modification of the method introduced in \citep{Cesa-Bianchi02}.
Due to space constraints we omit the description of the approach and just state the performance guarantee that we are able to prove.
%
%
\begin{theorem}\label{theorem:online-nonconvex}
For any $\delta \in [0,1]$ and $\beta > 0$, denote
\begin{equation}
    U_\delta(j, \beta) = 2 \sqrt{\frac{1}{s_{j,n}}\log \frac{s^3_{j,n}(s_{j,n}+1)}{\delta}} + \sqrt{\frac{1}{s_{j,n}} \log \frac{s_{j,n}^2}{\delta}}  + \sqrt{\frac{1}{s_{j,n}} \log \frac{s_{j,n}^2 \nN_\infty(\nL(\nH), \beta, n)}{\delta}} + 4\beta.
\end{equation}
If the subroutines of MACRO use the score-based online-to-batch conversion of \cite{Cesa-Bianchi02} with confidence $\delta$, it holds that
\begin{equation}
    \PP{R_n(h_n) - \inf_{h\in\nH}R_n(h) > W_{I_n,n}/s_{I_n,n} + U_\delta(I_n, \beta)} \leq k\delta / m + \PP{E_{k,m}^c}.
\end{equation}
\end{theorem}
The same conditions as before will ensure $\varepsilon$-learnability. 
Note that to perform this form of  online-to-batch conversion neither $k$ nor $m$ need to be known.

\begin{figure*}
    \centering
    \includegraphics[width=380pt]{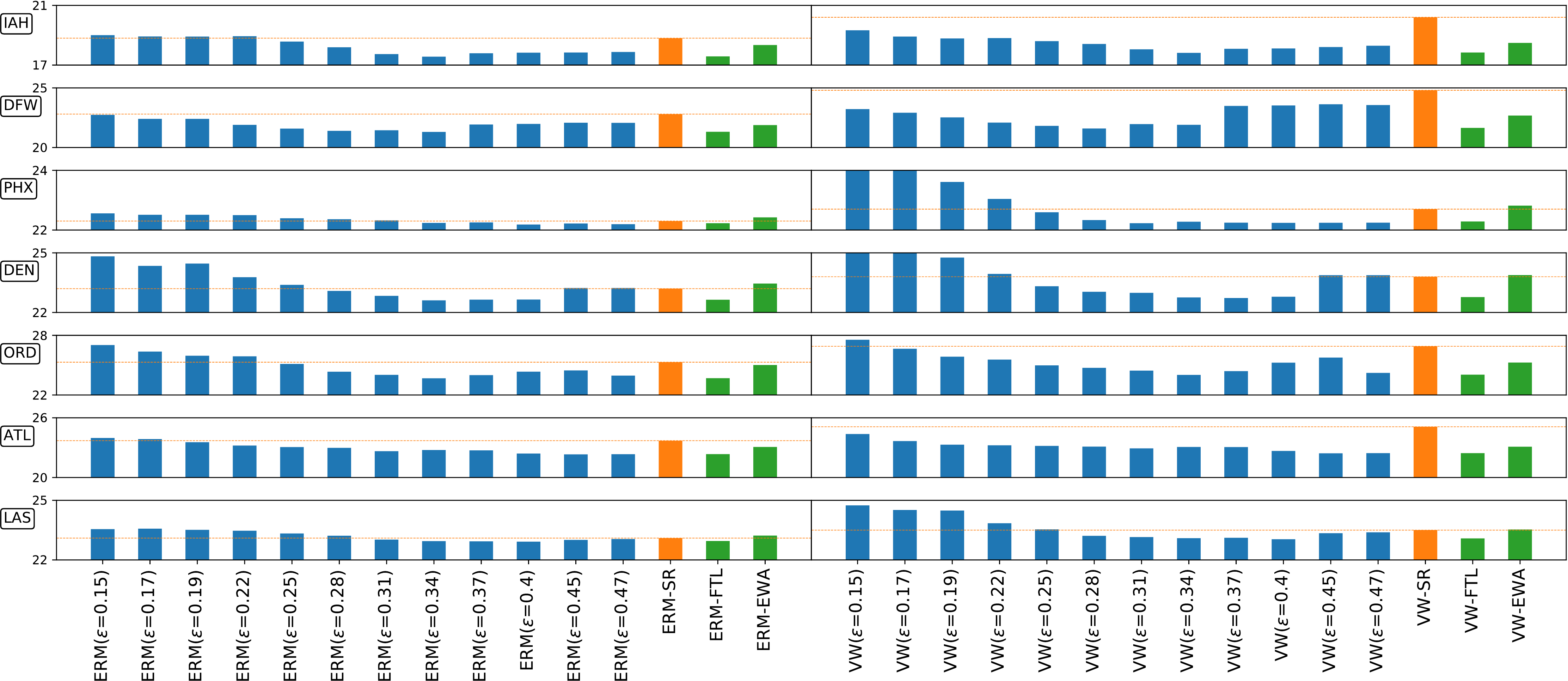}
    \caption{
Performance of MACRO with different subroutines on the DataExpo Airline dataset with the feature-based distance function.
Each row corresponds to a different airport labeled by its IATA code.
The y-axes shows error-rates; the x-axes is labeled by the short name of a subroutine and a threshold used in MACRO. 
ERM-FTL, ERM-EWA, VW-FTL and VW-EWA are the online strategies to choose the threshold.
Marginal versions of the subroutines, ERM-SR and VW-SR, act as baselines.}
\label{fig:airport}
\end{figure*}

\section{Experiments}\label{sec:practice}
In this section we highlight the practical applicability of the 
meta-algorithm by applying it to two large-scale sequential prediction problems, showing how CRM can lead to improved prediction quality compared to ordinary marginal learning. 
The code for all the experiments will be made publicly available. 

We adopt a classification setting, \ie $\nZ=\nX\times\nY$, where $\nX$ denotes a feature space, $\nY$ a set of labels, and $\ell$ is the $0/1$-loss.
Following the discussions of \cite{zimin2016aistats}, we use a distance between histories for the discrepancy bound. 
%
%
%
%
%
As the performance of the algorithm will depend on the choice of distance, we perform the experiments with two distances of different characteristics and study how they affect the predictive performance.
For the first distance, we consider only the labels of the data points in the histories and compare the vectors of the fractions of labels. 
For the second distance, we use the feature space and consider the $\ell_2$-distance between histories. 
The final conclusions for both distances are quite similar, therefore, we present the results only for the feature-based distance in the main manuscript. 
Results for the label-based distance as well as the exact definitions of the distances can be found in the supplementary material.

%
\begin{figure*}
    \centering
    \includegraphics[width=380pt]{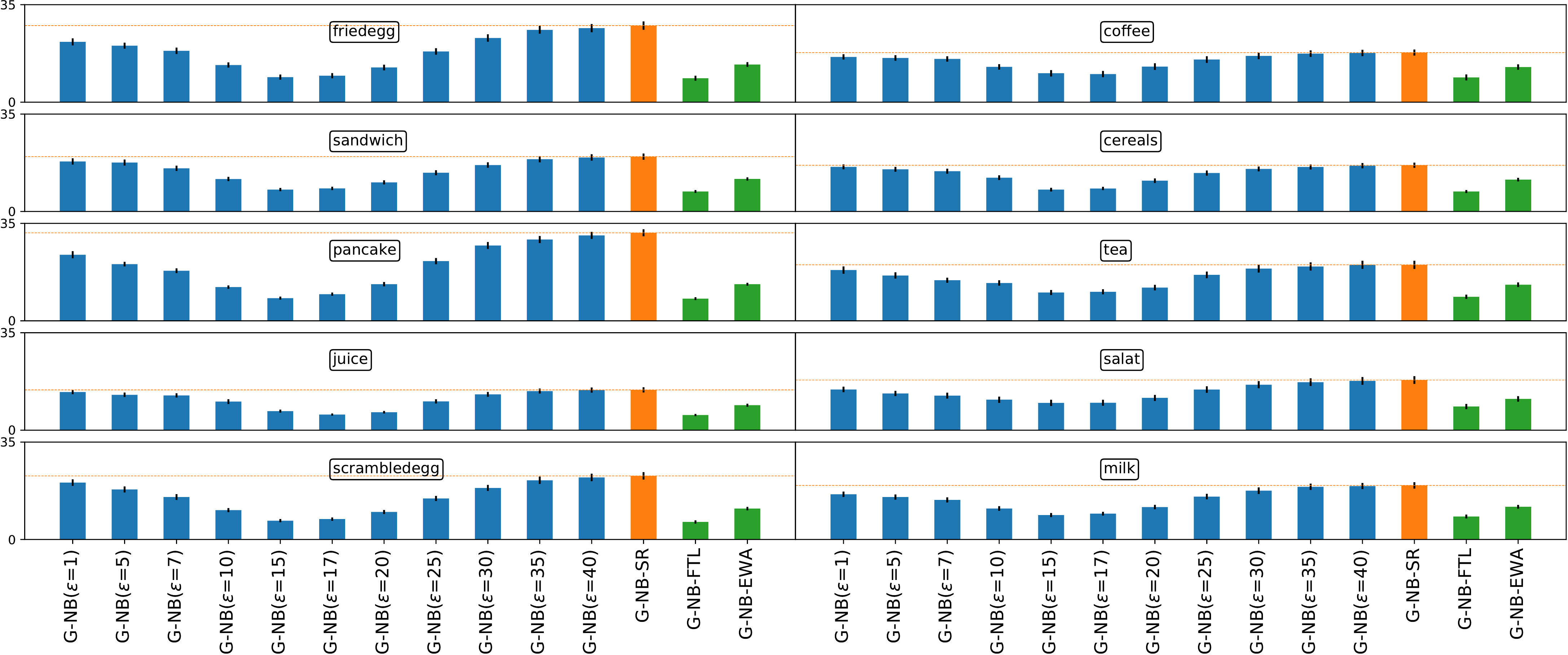}
    \caption{
Performance of MACRO with different subroutines on the Breakfast Actions dataset with feature-based distance function.
%
Each plot corresponds to different action. The y-axes shows error-rates averaged over the persons performing each action. 
The x-axes is labeled by the short name of a subroutine and a threshold used in MACRO.
G-NB-FTL and G-NB-EWA represent the online strategies to choose the threshold.
The baseline G-NB-SR is the marginal version of G-NB algorithm.}
\label{fig:breakfast}
\end{figure*}
%
%
%
%
%
%
\subsection{DataExpo Airline Dataset}
First, we apply MACRO to the \emph{DataExpo Airline} dataset \citep{airline_dataset},
%
which contains entries about all commercial flights in the United States between 1987 and 2008.
Out of these, we select the most recent year with complete data, 2007, and 
a number of the most active airports at that time, which gives, for example, more than 300000 flights for the Atalanta airport (ATL).
The task is a classification: predict if a flight is delayed ($y=1$) or not ($y=0$), where flights count as delayed if they arrive more than 15 minutes later than their 
scheduled arrival time. 
Clearly, the temporal order creates dependencies between flight delays
that a CRM approach can try to exploit for higher classification accuracy. 
Observations are defined by grouping the flights into 10 minute chunks,
so that at each time step, the task is to output a predictor that is applied to all flights in the next chunk. 

Since any algorithm can be used in MACRO as a subroutine, our goal is to show that MACRO is able to improve upon the baseline of just running the subroutine on the whole data (that any standard approach would do).
We perform experiments for both types of subroutines that we introduced in Section \ref{sec:subroutines} and which reflect the go-to choices for online classification problems in practice.
\begin{description}
\item[ERM] 
As tractable approximations for ERM we use logistic regression classifiers that are trained incrementally using stochastic
gradient descent, \ie $\omega_{t+1} \leftarrow \omega_t + \nabla_\omega \log \PPc{y_t}{x_t, \omega}$,
%
where $\omega_t$ are the parameters of the model at step~$t$.
\item[VW] As an online learning subroutine, we use \emph{Vowpal Wabbit}\footnote{\url{https://github.com/JohnLangford/vowpal_wabbit}}, 
a popular software package for large-scale online learning tasks. 
We set VW to use logistic loss as well with the default choice of meta parameters.
\end{description}

%
%
%

Figure~\ref{fig:airport} shows the results of the evaluating the MACRO with ERM and VW as subroutines comparing to a single ERM and VW algorithms run on the whole data.
Numeric results can be found in the supplemental material.
We see that in all of the presented airports MACRO achieves a better accuracy than the marginal versions of the corresponding algorithms for a wide range of thresholds $\varepsilon$. 
The effect is most profound with VW subroutine, where MACRO is able to achieve the performance on the level of MACRO with ERM subroutine, even though the VW subroutine itself seems to perform sub-optimally.

In addition to evaluating MACRO for a range of fixed thresholds, we show results for two methods that do not require to fix this parameter.
Both of them run a number of MACRO instances with different thresholds in parallel and choose the output by two standard online learning strategies: either Follow The Leader (FTL) or Exponentially Weighted Average (EWA).
Both strategies generally achieve good results, in particular better than marginal training,
with the online FTL strategy usually outperforming the EWA strategy and 
in all cases achieving an error-rate close to the best fixed threshold.
%
Even though both strategies use much more resources than a single instance of MACRO, 
they have the advantage of making the learning process completely parameter-free, 
and are therefore attractive 
if sufficient resources are available.
%

%
\subsection{Breakfast Actions Dataset}
In this set of experiments we present MACRO in a quite different setting.
We use the Breakfast Actions Dataset\footnote{\url{http://serre-lab.clps.brown.edu/resource/breakfast-actions-dataset/}}, which consists of videos of 52 people performing 10 actions related to breakfast preparation.
Each combination of a person and an action is treated as a separate learning task and the performance is measured by per frame error rate.
Following the usage of a Gaussianity assumption by previous approaches \cite{kuehne2014language,kuehne2016end}, we use Gaussian Naive Bayes classifiers trained online as subroutines.
\begin{description}
\item[G-NB] 
The algorithm tracks the running average in the feature space for each class separately and predicts the class with the closest mean.
After receiving a new point, the algorithm incrementally updates the mean of the corresponding class.
\end{description}
As for the airports dataset, we present the results for the feature-based distance, while results for the label-based one can be found in the supplement.
As above, we also evaluate the FTL and EWA strategies for threshold selection.

The results are presented in Figure \ref{fig:breakfast}.
We observe that the effect of MACRO is even stronger than for the airport dataset.
The error-rate is always reduced, in some cases by more than 70\% relatively to the baseline.
Both threshold-selection strategies show excellent performance, with FTL again outperforming EWA.

Overall, we see that MACRO consistently outperforms the traditional online algorithms
for both datasets.
This illustrates two facts: CRM is indeed a promising approach to sequential prediction problems,
and MACRO allows applying CRM principles to large real-world datasets that previously suggested methods are unable to handle.


%
%

%
%
\section{Conclusion}
In this paper we presented a new meta-algorithm, MACRO, for conditional risk minimization that is based on the idea of maintaining a number of learning subroutines that are created when necessary on-the-fly and trained individually only on relevant subsets of data.
We proved theoretical guarantees on the performance of the presented meta-algorithm for different choices of subroutines.
In contrast to previous work, MACRO does not require storing all observed data and can be efficiently implemented.
This makes MACRO the first CRM algorithm that is able to handle sequential learning problems of practically relevant size, as we demonstrate by applying it to two large scale problems, the DataExpo Airline and the Breakfast Actions datasets.
%
%

%

%
%

%

\bibliography{biblio}
\bibliographystyle{plain}

\newpage
\setcounter{theorem}{3}
\section*{Supplementary material}
To characterize a complexity of some function class we use covering numbers and a sequential fat-shattering dimension.
But before we could give those definitions, we need to introduce a notion of $\nZ$-valued trees.

A $\nZ$-valued tree of depth $n$ is a sequence $z_{1:n}$ of mappings $z_i: \lbrace \pm 1 \rbrace^{i-1} \ra \nZ$.
A sequence $\varepsilon_{1:n} \in \lbrace \pm 1 \rbrace^{n}$ defines a path in a tree.
To shorten the notations, $\bz_t(\varepsilon_{1:t-1})$ is denoted as $\bz_t(\varepsilon)$.
For a double sequence $z_{1:n}, z'_{1:n}$, we define $\chi_t(\varepsilon)$ as $z_t$ if $\varepsilon = 1$ and $z'_t$ if $\varepsilon = -1$.
Also define distributions $p_t(\varepsilon_{1:t-1}, z_{1:t-1}, z'_{1:t-1})$ over $\nZ$ as $\PPc{\cdot}{\chi_1(\varepsilon_1), \dots, \chi_{t-1}(\varepsilon_{t-1})}$, where $\mathbb{P}$ is a distribution of a process under consideration.
Then we can define a distribution $\rho$ over two $\nZ$-valued trees $\bz$ and $\bz'$ as follows: $\bz_1$ and $\bz'_1$ are sampled independently from the initial distribution of the process and for any path $\varepsilon_{1:n}$ for $2 \leq t \leq n$, $\bz_t(\varepsilon)$ and $\bz'_t(\varepsilon)$ are sampled independently from $p_t(\varepsilon_{1:t-1}, \bz_{1:t-1}(\varepsilon), \bz'_{1:t-1}(\varepsilon))$.

For any random variable $\by$ that is measurable with respect to $\sigma_{n}$ (a $\sigma$-algebra generated by $\bz_{1:n}$), we define its symmetrized counterpart $\tilde{\by}$ as follows.
We know that there exists a measurable function $\psi$ such that $\by = \psi(\bz_{1:n})$.
Then we define $\tilde{\by} = \psi(\chi_1(\epsilon_1), \dots, \chi_n(\varepsilon_n))$, where the samples used by $\chi_t$'s are understood from the context.

Now we can define covering numbers.
\begin{definition}
A set, $V$, of $\nR$-valued trees of depth $n$ is a (sequential) \textbf{$\theta$-cover} 
(with respect to the $\ell_\infty$-norm) of $\nF \subset \left\lbrace f: \nZ \ra \nR \right\rbrace$ on a tree $\bz$ of depth $n$ if
\begin{align}
\forall f \in \nF, & \forall \varepsilon \in \lbrace \pm 1\rbrace^n, \exists v\in V: \\
& \max_{1\leq t \leq n} \abs{ f(\bz_t(\varepsilon)) - v_t(\varepsilon) } \leq \theta.
\end{align}
The (sequential) \textbf{$\theta$-covering number} of a function class $\nF$ on a given tree $\bz$ is
\begin{align}\nN_\infty(\nF, \theta, \bz ) = \min \lbrace & \abs{V}: V 
\text{ is an $\theta$-cover} \\ 
&\text{w.r.t. $\ell_\infty$-norm of $\nF$ on $\bz$} \rbrace.
\end{align}
The \textbf{maximal $\theta$-covering number} of a function class $\nF$ over depth-$n$ trees is
\begin{equation}
\nN_\infty(\nF, \theta, n ) = \sup_{\bz}\nN_\infty(\nF, \theta, \bz ).
\end{equation}
\end{definition}

To control the growth of covering numbers we use the following notion of complexity.
\begin{definition}
A $\nZ$-valued tree $\bz$ of depth $n$ is $\theta$-shattered by a function class $\nF \subseteq \lbrace f: \nZ \ra \nR \rbrace$ if there exists an $\nR$-valued tree $s$ of depth $n$ such that
\begin{align}
\forall \varepsilon \in \lbrace \pm 1 \rbrace^n, & \exists f \in\nF \text{ s.t. } 1 \leq t \leq n,\\
& \varepsilon_t(f(\bz_t(\varepsilon)) - s_t(\varepsilon)) \geq \theta/2.
\end{align}
The \emph{(sequential) fat-shattering dimension} $\text{fat}_\theta(\nF)$ at scale $\theta$ is the largest $d$ such that $\nF$ $\theta$-shatters a $\nZ$-valued tree of depth $d$.
\end{definition}
An important result of \citep{Rakhlin01} is the following connection between the covering numbers and the fat-shattering dimension.
\begin{lemma}[Corollary 1 of \citep{Rakhlin01}]\label{lemma:coveringbound}
Let $\nF \subseteq \lbrace f : \nZ \ra [-1, 1] \rbrace$. For any $\theta > 0$ and any $n \geq 1$, we have that
\begin{equation}
\nN_\infty(\nF, \theta, n) \leq \left(\frac{2en}{\theta}\right)^{\text{fat}_\theta(\nF)}.
\end{equation}
\end{lemma}

In all of the proofs we use the following technical lemma about the meta-algorithm.
\begin{lemma}\label{lemma:technical}
Irrespectively of the subroutine used by the meta-algorithm, for any $\alpha \in [0, 1]$ and $\beta \in [0, \beta/4]$, we have
\begin{align}
& \PP{ \sup_{h\in\nH}\absd{\frac{1}{s_{I_n,n}} \sum_{t\in C_{I_n,n}} (\ell(h, \bz_t) - R_{t-1}(h)) } > \alpha \wedge E_{k,m} } \notag \\
& \leq \frac{2k\nN_\infty(\nL(\nH),\beta, n)}{(\alpha-4\beta)^2} e^{-\frac{1}{2}m(\alpha-4\beta)^2}.
\end{align}
Moreover, for any $\alpha \in [0, 1]$ with $g_{I_n, i} = \ell(h_{I_n,i-1}, \bz_{t_{I_n,i}}) - R_{t_{I_n,i}-1}(h_{I_n,i-1})$
\begin{align}
& \PP{ \absd{\frac{1}{s_{I_n,n}} \sum_{i=1}^{s_{I_n,n}} g_{I_n, i} } > \alpha \wedge E_{k,m} } \notag \leq \frac{2k}{\alpha^2} e^{-\frac{1}{2}m\alpha^2}.
\end{align}

\end{lemma}
\begin{proof}
Introduce events $A_i = \lbrace I_n = i \rbrace$ for $i = 1, \dots, k$ and $B_{i,j} = \lbrace s_{j,n} = i \rbrace$ (we suppress the dependence on $n$ to increase readability).
Observe that $E_{k,m} = \lbrace \cup_{i\geq 1} A_i \rbrace \wedge \lbrace \cup_{i \geq m} B_{i,I_n} \rbrace$.
Denoting $\Lambda(j) =  \sup_{h\in\nH}\absd{\frac{1}{s_{j,n}} \sum_{t\in C_{j,n}} (\ell(h, \bz_t) - R_{t-1}(h)) }$, we have
\begin{align}
& \PP{\Lambda(I_n) > \alpha \wedge E_{k,m}} \\
& \leq \sum_{j \in \supp{I_n}} \PP{\Lambda(j) > \alpha \wedge \lbrace \cup_{i \geq m} B_{i,j} \rbrace}.
\end{align}
Each of the last probabilities can be bounded using a union bound.
\begin{align}
\PP{\Lambda(j) > \alpha \wedge \lbrace \cup_{i \geq m} B_{i,j} \rbrace} \leq \sum_{i\geq m} \PP{\Lambda(j) > \alpha \wedge B_{i,j}}.
\end{align}
Using Lemma 4 from \citep{zimin2016aistats}, we get
\begin{align}
\PP{\Lambda(j) > \alpha \wedge B_{i,j}} \leq 2 \nN_\infty(\nL(\nH), \beta, n) e^{-\frac{1}{2}i(\alpha-4\beta)^2}.
\end{align}
Summing the probabilities, we obtain the first statement of the lemma.

For the second part of the lemma, denote $\Lambda(j) = \absd{\frac{1}{s_{j,n}} \sum_{i=1}^{s_{j,n}} (\ell(h_{j,i-1}, \bz_{t_{j,i}}) - R_{t_{j,i}-1}(h_{j,i-1})) }$ and, using the same decomposition as above, we have to bound $\PP{\Lambda(j) > \alpha \wedge B_{i,j}}$.
Observe that each $h_{j,i}$ is adapted to the filtration generated by $\lbrace \bz_{t_{j,i}} \rbrace_{i=1}^\infty$, hence, $\ell(h_{j,i}, \bz_{t_{j,i-1}}) - R_{t_{j,i}-1}(h_{j,i-1})$ behaves like a martingale difference sequence. 
However, there is a technical difficulty in the fact that the indices $t_{j,i}$ are in fact stopping times.
To get around it, observe that we can write $\Lambda(j)$ as a sum over all the data with the adapted weights.
Set $w_t$ to $1$ if we updated the algorithm $j$ at step $t$ and to $0$ otherwise.
Correspondingly, define $\bar{h}_t$ as the last chosen hypothesis by the $j$-th algorithm.
This way, both $w_t$ and $\bar{h}_t$ are adapted to the original process.
Then
\begin{align}
& \sum_{i=1}^{s_{j,n}} ( \ell(h_{j,i-1}, \bz_{t_{j,i}}) - R_{t_{j,i}-1}(h_{j,i-1}) )\\
& = \sum_{t=1}^{n} w_t ( \ell(\bar{h}_t, \bz_t) - R_{t-1}(\bar{h}_t)).
\end{align}
At this point we can again use Lemma 4 from \citep{zimin2016aistats} and get the second statement of the lemma.
\end{proof}

\begin{proof}[Proof of Lemma 1]
The lower bound comes from the fact that MACRO constructs an $\varepsilon$-covering. For the upper bound, observe that a new subroutine is started if and only if its associated conditional distribution differs by more than 
$\varepsilon$ from the ones of all previously created subroutines. 
Therefore, the set of conditional distribution associated with 
subroutines form an \emph{$\varepsilon$-separated set} with respect to $M_{i,j}$'s (no two elements 
are closer than $\varepsilon$ to each other). The maximal size 
of such a set is at most the \emph{covering number} of half 
the distance. 
\end{proof}

\begin{proof}[Proof of Theorem 1]
We start by the usual argument for the empirical risk minimization that allows us to focus on the uniform deviations.
\begin{equation}
R_n(h_n) - \inf_{h\in\nH}R_n(h) \leq 2 \sup_{h\in\nH}\absd{R_n(h) - \hat{R}_n(h, I_n)}.
\end{equation}
Denoting by $\bar{R}_n(h, I_n) = \frac{1}{s_{I_n,n}}\sum_{t\in C_{I_n,n}} R_{t-1}(h)$, we can upper bound the last term.
\begin{align}
& \sup_{h\in\nH}\absd{R_n(h) - \hat{R}_n(h, I_n)} \\
& \leq \sup_{h\in\nH}\absd{R_n(h) - \bar{R}_n(h, I_n)} + \sup_{h\in\nH}\absd{\bar{R}_n(h, I_n) - \hat{R}_n(h, I_n)} \\
& \leq \frac{1}{s_{I_n,n}}\sum_{i\in C_{I_n,n}} d_{i,n+1} + \sup_{h\in\nH}\absd{\bar{R}_n(h, I_n) - \hat{R}_n(h, I_n)} \\
& \leq \frac{1}{s_{I_n,n}}\sum_{i\in C_{I_n,n}} M_{i,n+1} + \sup_{h\in\nH}\absd{\bar{R}_n(h, I_n) - \hat{R}_n(h, I_n)} \\
& \leq 2\varepsilon + \sup_{h\in\nH}\absd{\bar{R}_n(h, I_n) - \hat{R}_n(h, I_n)},
\end{align}
where the last bound follows from the way the meta-algorithm chooses $I_n$.
Hence, we get
\begin{align}
& \PP{R_n(h_n) - \inf_{h\in\nH}R_n(h) > \alpha + 2\varepsilon} \\
& \leq \PP{ \sup_{h\in\nH}\absd{\bar{R}_n(h, I_n) - \hat{R}_n(h, I_n)} > \alpha }.
\end{align}
The last probability can be bounded using Lemma \ref{lemma:technical} giving us the statement of the theorem.
\end{proof}

\begin{proof}[Proof of Theorem 2]
Note that by the way $I_n$ is chosen, we get for any $h \in \nH$ that
\begin{equation}
R_n(h) - R_{t_{I_n,i-1}}(h) \leq 2\varepsilon.
\end{equation}
Therefore, by using the convexity of the loss
\begin{align}
R_n(h_n) & \leq \frac{1}{s_{I_n,n}}\sum_{i=1}^{s_{I_n,n}} R_n(h_{I_n,i}) \\
& \leq \frac{1}{s_{I_n,n}}\sum_{i=1}^{s_{I_n,n}} R_{t_{I_n,i}-1}(h_{I_n,i}) + 2\varepsilon.
\end{align}
Similarly, for any fixed $h$
\begin{equation}
R_n(h) \geq \frac{1}{s_{I_n,n}}\sum_{i=1}^{s_{I_n,n}} R_{t_{I_n,i}-1}(h) - 2\varepsilon.
\end{equation}
Therefore,
\begin{align}
& R_n(h_n) - \inf_{h\in\nH}R_n(h) \\
& \leq 4\varepsilon + \frac{1}{s_{I_n,n}}\sum_{i=1}^{s_{I_n,n}} R_{t_{I_n,i}-1}(h_{I_n,i}) \\
& - \inf_{h\in\nH} \frac{1}{s_{I_n,n}}\sum_{i=1}^{s_{I_n,n}} R_{t_{I_n,i}-1}(h). \label{eq:risk-difference}
\end{align}
We split the last difference into the following three terms and deal with them separately.
\begin{align}
T_1 & = \frac{1}{s_{I_n,n}}\sum_{i=1}^{s_{I_n,n}}( R_{t_{I_n,i}-1}(h_{I_n,i}) - \ell(h_{I_n,i}, \bz_{t_{I_n,i}})) \\
T_2 & = \frac{1}{s_{I_n,n}} \sum_{i=1}^{s_{I_n,n}} \ell(h_{I_n,i}, \bz_{t_{I_n,i}}) - \inf_{h\in\nH} \sum_{i=1}^{s_{I_n,n}} \ell(h, \bz_{t_{I_n,i}}) \\
T_3 & = \inf_{h\in\nH} \sum_{i=1}^{s_{I_n,n}} \ell(h, \bz_{t_{I_n,i}}) - \inf_{h\in\nH} \frac{1}{s_{I_n,n}}\sum_{i=1}^{s_{I_n,n}} R_{t_{I_n,i}-1}(h).
\end{align}
The first term can be bounded using Lemma \ref*{lemma:technical}.
$T_2$ is in fact just $W_{I_n,n}$.
For $T_3$ observe that
\begin{align}
& \inf_{h\in\nH} \frac{1}{s_{I_n,n}}\sum_{i=1}^{s_{I_n,n}} R_{t_{I_n,i}-1}(h) \geq  \inf_{h\in\nH}\frac{1}{s_{I_n,n}}\sum_{i=1}^{s_{I_n,n}} \ell(h, \bz_{t_{I_n,i}}) \\
& + \inf_{h\in\nH} ( \frac{1}{s_{I_n,n}}\sum_{i=1}^{s_{I_n,n}} ( R_{t_{I_n,i}-1}(h) - \ell(h, \bz_{t_{I_n,i}}) ).
\end{align}
Therefore, $T_3$ is bounded by $\tilde{T}_3$:
\begin{equation}
\tilde{T}_3 = \sup_{h\in\nH} ( \frac{1}{s_{I_n,n}}\sum_{i=1}^{s_{I_n,n}} ( \ell(h, \bz_{t_{I_n,i}}) - R_{t_{I_n,i}-1}(h) ).
\end{equation}
Combining everything together,
\begin{align}
& \PP{R_n(h_n) - \inf_{h}R_n(h) > \alpha + 4 \varepsilon + W_{I_n,n}} \\
& \leq \PP{ T_1 + \tilde{T}_3 > \alpha  \wedge E_{k,m} } + \PP{E_{k,m}^c} \\
& \leq \PP{ T_1 > \alpha/2 \wedge E_{k,m} }\\
&  + \PP{ \tilde{T}_3 > \alpha/2 \wedge E_{k,m} } + \PP{E_{k,m}^c} \notag
\end{align}
The both terms in the last line can be bounded using Lemma \ref{lemma:technical} giving us the statement of the theorem.
\end{proof}

\paragraph{Online-to-batch conversion for non-convex losses.}
Here we describe the modification of the online-to-batch conversion method of \cite{Cesa-Bianchi02}.
As the original method was designed to work for \iid data, we need to extend it to stochastic processes.
The general idea is to assign a score to each of $h_{j,i}$ 
and choose the one with the lowest score.
For a given confidence $\delta > 0$, the score of $h_{j,i}$ is computed as
\begin{equation}
u_n(j, i) = \tilde{R}_{n}(j,i) + c_{j,\delta}(s_{j,n}-i).
\end{equation}
where
\begin{equation}
\tilde{R}_{n}(j,i) = \frac{1}{s_{j,n}-i}\sum_{k=i+1}^{s_{j,n}}\ell(h_{j,i}, \bz_{t_{j,k}}) \quad \text{and} \quad c_{j, \delta}(t) = \sqrt{\frac{1}{2(t+1)}\log \frac{s^3_{j,n}(s_{j,n}+1)}{\delta}}.
\end{equation}
Setting $J_n = \argmin_{1\leq i \leq s_{I_n,n}} u_n(I_n, i)$, MACRO's output is $h_n = h_{I_n, J_n}.$

The following lemma is analog of Lemma 3 from \cite{Cesa-Bianchi02} proved for the case of dependent data and the conditional risk.
\begin{lemma}\label{lemma:nonconvex}
For the setting of Theorem 3, let 
\begin{equation}
v(j,i) =  R_n(h_{j,i}) + 2 c_{I_n, \delta}(s_{I_n,n}-i).
\end{equation}
Then we have
\begin{equation}
\PP{R_n(h_n) > \min_{1 \leq i \leq s_{I_n,n}} v(I_n, i) + 2 \varepsilon \wedge E_{k,m}} \leq \frac{k \delta}{m}.
\end{equation}
\end{lemma}
\begin{proof}
Introduce events $A_{r} = \lbrace \abs{\supp{I_n}} \leq k \wedge s_{I_n, n} = r\rbrace$.
Using a union bound, we have
\begin{align}
& \PP{R_n(h_n) > \min_{1 \leq i \leq s_{I_n,n}} v(I_n, i) + 2 \varepsilon \wedge E_{k,m}} \\
& \leq \sum_{r \geq m} \PP{R_n(h_n) > \min_{1 \leq i \leq s_{I_n,n}} v(I_n, i) + 2 \varepsilon \wedge A_r}.
\end{align}
Therefore, we will focus on the last probabilities.
Let $J^\star_n = \argmin_{1 \leq i \leq s_{I_n,n}} v(I_n, i) $ and also introduce events $B_i = \lbrace \tilde{R}_n(I_n, i) + c_{I_n,\delta}(s_{I_n,n}-i) \leq \tilde{R}_n(I_n, J^\star_n) + c_{I_n,\delta}(s_{I_n,n}-J^\star_n) \rbrace$.
Then, since $\tilde{R}_n(I_n, J_n) + c_{I_n,\delta}(s_{I_n,n}-J_n) \leq \tilde{R}_n(I_n, J^\star_n) + c_{I_n,\delta}(s_{I_n,n}-J^\star_n)$ is always true, we get
\begin{align}
& \PP{R_n(h_n) > \min_{1 \leq i \leq s_{I_n,n}} v(I_n, i) + 2 \varepsilon \wedge A_r} \\
& \leq \sum_{i=1}^{r} \PP{R_n(h_{I_n,i}) > v(I_n, J^\star_n) + 2 \varepsilon \wedge B_i \wedge A_r}.
\end{align}
Observe that if $B_i$ is true, then at least one of the following events is also true.
\begin{align}
& D_{1,i} = \lbrace \tilde{R}_n(I_n, i) \leq R_n(h_{I_n,i}) - \varepsilon - c_{I_n, \delta}(s_{I_n,n} - i) \rbrace,\\
& D_{2,i} = \lbrace R_n(h_{I_n,i}) < v(I_n, J^\star_n) + 2\varepsilon  \rbrace, \\
& D_3 = \lbrace \tilde{R}_n(I_n, J^\star_n) > R_n(h_{I_n,J^\star_n}) + \varepsilon + c_{I_n, \delta}(s_{I_n,n} - J^\star_n) \rbrace.
\end{align}
From this we get
\begin{align}
& \PP{R_n(h_{I_n,i}) > v(I_n, J^\star_n) + 2 \varepsilon \wedge B_i \wedge A_r} \\
& \leq \PP{ R_n(h_{I_n,i}) > v(I_n, J^\star_n) + 2 \varepsilon \wedge D_{1,i} \wedge A_r } \\
& + \PP{ R_n(h_{I_n,i}) > v(I_n, J^\star_n) + 2 \varepsilon \wedge D_{2,i} \wedge A_r } \\
& + \PP{ R_n(h_{I_n,i}) > v(I_n, J^\star_n) + 2 \varepsilon \wedge D_{3} \wedge A_r }.
\end{align}
First, notice that 
\begin{equation}
\PP{ R_n(h_{I_n,i}) >v(I_n, J^\star_n) + 2 \varepsilon \wedge D_{2,i} \wedge A_r } = 0.
\end{equation}
Moreover, since 
\begin{equation}
\abs{ R_n(h_{I_n,i}) - \bar{R}(I_n,i)} \leq 2 \varepsilon
\end{equation}
for $\bar{R}(j,i) = \frac{1}{s_{j,n}-i}\sum_{s=i+1}^{s_{j,n}} R_{t_{j,s}-1}(h_{j,i})$
we have
\begin{align}
& \PP{D_{1,i} \wedge A_r} \\
& = \PP{  \tilde{R}_n(I_n, i) \leq R_n(h_{I_n,i}) - 2\varepsilon - c_{I_n, \delta}(s_{I_n,n} - i) \wedge A_r } \\
& \leq \PP{ \tilde{R}_n(I_n, i) \leq \bar{R}(I_n,i) - c_{I_n, \delta}(s_{I_n,n} - i) \wedge A_r } \\
& \leq \sum_{j\in\supp{I_n}} \PP{ \tilde{R}_n(j, i) \leq \bar{R}(j,i) - c_{j, \delta}(s_{j,n} - i) \wedge A_r }.
\end{align}
From Lemma \ref{lemma:technical} we get that
\begin{align}
& \PP{ \tilde{R}_n(j, i) \leq \bar{R}(j,i) - c_{j, \delta}(s_{j,n} - i) \wedge A_r } \\
& \leq \frac{\delta}{r^3(r+1)}.
\end{align}
And, hence, 
\begin{equation}
\PP{D_{1,i} \wedge A_r} \leq \frac{k\delta}{r^3(r+1)}.
\end{equation}
Similarly,
\begin{equation}
\PP{D_{3} \wedge A_r} \leq \frac{k\delta}{r^2(r+1)}.
\end{equation}
Combining these two together, we get
\begin{equation}
\PP{R_n(h_n) > \min_{1 \leq i \leq s_{I_n,n}} v(I_n, i) + 2 \varepsilon \wedge A_r} \leq \frac{k\delta}{r^2},
\end{equation}
which gives us the statement on the lemma.
\end{proof}
\begin{proof}[Proof of Theorem 3]
From Lemma \ref{lemma:nonconvex} we get that with high probability
\begin{equation}
R_n(h_n) \leq \min_{1 \leq i \leq s_{I_n,n}} v(I_n, i) + 2 \varepsilon.
\end{equation}
Hence, we focus on bounding $\min_{1 \leq i \leq s_{I_n,n}} v(I_n, i)$.
Observe that
\begin{align}
& \min_{1 \leq i \leq s_{I_n,n}} v(I_n, i) \\
& \leq \frac{1}{s_{I_n,n}} \sum_{i=1}^{s_{I_n, n}} ( R_n(h_{I_n,i}) + c_{I_n,\delta/2}(s_{I_n,n}-i) ) \\
& \leq \frac{1}{s_{I_n,n}} \sum_{i=1}^{s_{I_n, n}} ( R_{t_{I_n,i}-1}(h_{I_n,i}) + 2\varepsilon + c_{I_n,\delta/2}(s_{I_n,n}-i) ) \\
& \leq \frac{1}{s_{I_n,n}} \sum_{i=1}^{s_{I_n, n}} R_{t_{I_n,i}-1}(h_{I_n,i}) + 2\varepsilon \\
& + 2 \sqrt{\frac{1}{s_{I_n,n}}\log \frac{s^3_{I_n,n}(s_{I_n,n}+1)}{\delta}}.
\end{align}
Similarly to Lemma \ref{lemma:technical}, we have that with high probability on $E_{k,m}$:
\begin{align}
& \frac{1}{s_{I_n,n}} \sum_{i=1}^{s_{I_n, n}} R_{t_{I_n,i}-1}(h_{I_n,i}) \\
& \leq \frac{1}{s_{I_n,n}} \sum_{i=1}^{s_{I_n, n}} \ell(h_{I_n,i}, \bz_{t_{I_n,i}}) + \sqrt{\frac{1}{s_{I_n,n}} \log \frac{s_{I_n,n}^2}{\delta}}.
\end{align}
Similarly to the proof of Theorem 2, we get with high probability on $E_{k,m}$:
\begin{align}
& - \inf_{h\in\nH} R_n(h) \leq - \inf_{h} \frac{1}{s_{I_n,n}} \sum_{i=1}^{s_{I_n, n}} \ell(h, \bz_{t_{I_n, i}}) \\
& + \sqrt{\frac{1}{s_{I_n,n}} \log \frac{s_{I_n,n}^2 \nN_\infty(\nL(\nH), \beta, n)}{\delta}} + 4\beta.
\end{align}
Therefore, we can conclude that
\begin{align}
& \PP{ R_n(h_n) - \inf_{h\in\nH} R_n(h) > W_{I_n,n} + U_\delta(I_n, \beta) \wedge E_{k,m}} \\
& \leq \frac{3k\delta}{m} + \PP{E_{k,m}^c}.
\end{align}
\end{proof}
\paragraph{Full learnability.}
First, we fix a sequence, $\varepsilon_n$, that converges to zero.
Then, at a step $n$ of MACRO, when it decides on the subroutines to update 
and chooses the active hypothesis, we perform all the necessary checks using 
the corresponding element of the sequence, \ie $\varepsilon_n$.
This change results in the following version of Theorem 2 
(and analogous results can be derived from Theorems 1 and 3).
\begin{theorem}\label{theorem:online-convex-learnability}
For a convex loss $\ell$, if the subroutines of the meta-algorithm use an averaging for online-to-batch conversion, we have for any $\alpha \in [0,1]$ and $\beta \in [0, \alpha/8]$
\begin{align}
& \PP{R_n(h_n) - \inf_{h\in\nH}R_n(h) > \alpha + \frac{W_{I_n,n}}{s_{I_n,n}} + \frac{4}{s_{I_n,n}} \sum_{t \in C_{I_n,n}} \varepsilon_t } \notag \\
& \leq \frac{4k\nN_\infty(\nL(\nH),\beta, n)}{(\alpha/2-4\beta)^2} e^{-\frac{1}{2}m(\alpha/2-4\beta)^2} + \PP{E_{k,m}^c}.
\end{align}
\end{theorem}
The same conditions as above ensure the learnability for this modified version of the meta-algorithm.
The described change also affects the number of subroutines started by MACRO.
Following the same argument as in Lemma 1, this number is bounded by $\nN(M, n, \varepsilon_n/2)$, thereby, implying a trade-off between the accuracy of the approximations and the number of subroutines, which is equivalent to the amount of computations we do at each turn.
\begin{proof}[Proof of Theorem 4]
The only part that changes in the proof is the first step when we approximate the risk at step $n$ by the average of risks.
Then, instead of $2\varepsilon$ we get $2 \frac{1}{s_{I_n,n}} \sum_{i=1}^{s_{I_n,n}} \varepsilon_{t_{I_n,i}}$.
\end{proof}

\paragraph{Experiments}
The way we compute distances between histories differ between the two studied settings.
For the Airports dataset, since the number of flights in each time stamp changes, we compute the feature-based distance as the approximate bottleneck distance:
to compare to sets of vectors, $S$ and $T$, we compute all pairwise $\ell_2$-distances between the elements of
$S$ and $T$, take the smallest $\max\lbrace \abs{S}, \abs{T} \rbrace$ ones and compute their average.
Denoting this approximate bottleneck distance as $\bar{D}_1$, the final distance between chunks 
is computed as
\begin{equation}
D_1(S, T) = \frac{1}{2}\bar{D}_1(S^0, T^0) + \frac{1}{2}\bar{D}_1(S^1, T^1),
\end{equation}
where $S^y$ and $T^y$ are the subsets of $S$ and $T$ with label~$y$.
For the Breakfast dataset, we fix a finite length history and compute the $\ell_2$-distance between the vectors on the same positions and take the average.
In the experiments, we used the distance of length 5, however, we tried out over values and found that the results are not very sensitive to the actual length.

For the second distance, $D_2$, we only make use of the labels of the points in the histories. 
For any history $S$, define $p(S)=(p_1,\dots, p_K)^T$ with $p_i$ being the fraction of the class $i$ in $S$ and $K$ being the number of classes.
Then we set $D_2(S, T) = \|p(S)-p(T)\|^2$. 
The theoretical analysis is oblivious to the fact how we initialize the subroutines.
In experiments, whenever we start a new subroutine, we give it a warm start by initializing it with the parameters of the closest subroutine in terms of discrepancies.
The results for MACRO with label-based distance on the Airports dataset are presented in Figure~\ref{fig:airports-label} and on the Breakfast dataset are in Figure~\ref{fig:breakfast-label}.

In addition to the plots, we proide the exact error-rates for all the presented experiments in Tables~\ref{fig:airports-bottleneck-nums}, \ref{fig:airports-label-nums}, \ref{fig:breakfast-bottleneck-nums} and \ref{fig:breakfast-label-nums}.

\begin{figure*}
    \centering
    \includegraphics[scale=0.3]{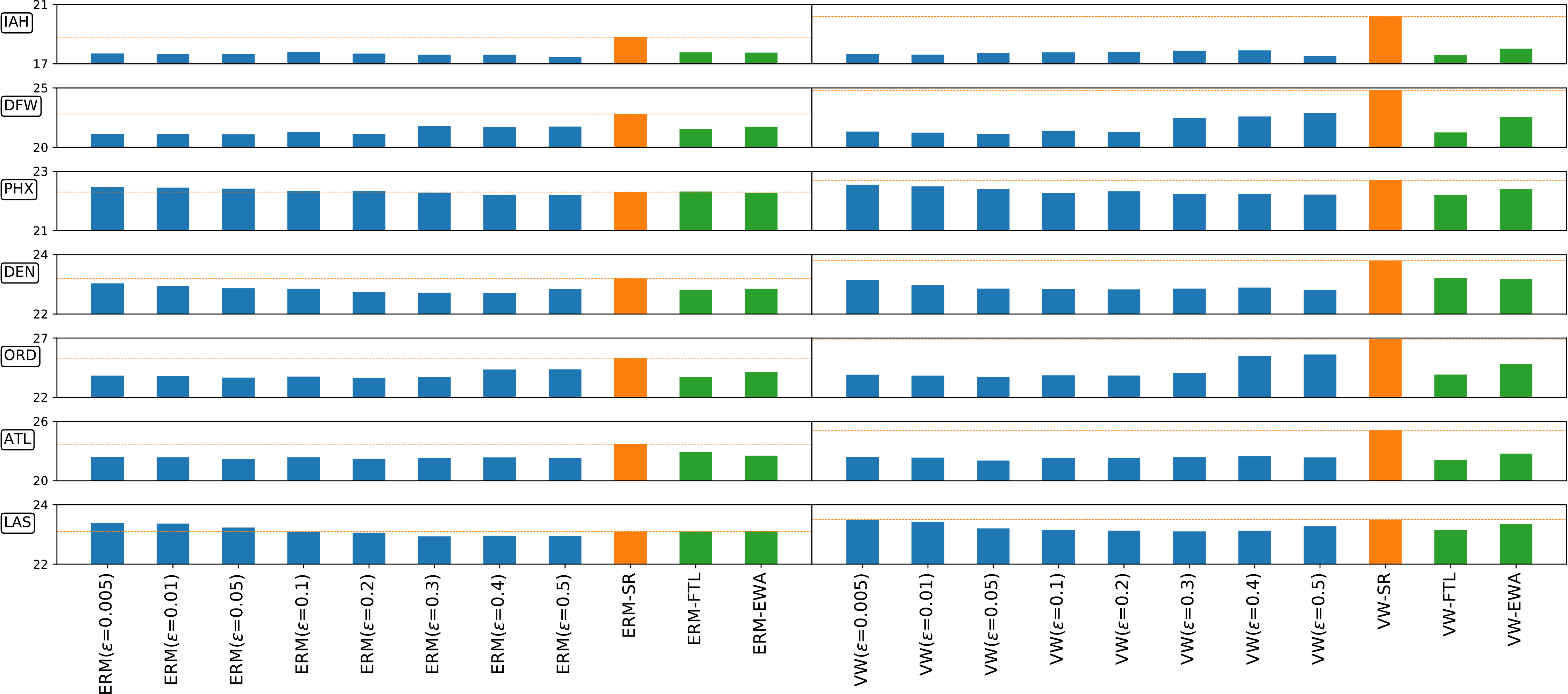}
    \caption{
Performance of MACRO with different subroutines on the DataExpo Airline dataset with the label-based distance function.
%
Each row of plots corresponds to a different airport labeled by its IATA code.
The y-axes shows error-rates in percents and the x-axes is labeled by the short name of a subroutine and a threshold used in MACRO.
ERM-FTL, VW-EWA, ERM-FTL and VW-EWA represent the online strategies to choose the threshold.
ERM-SR and VW-SR are the marginal versions of the subroutine algorithms and act as a baselines.
}
    \label{fig:airports-label}
\end{figure*}

\begin{figure*}
    \centering
    \includegraphics[scale=0.3]{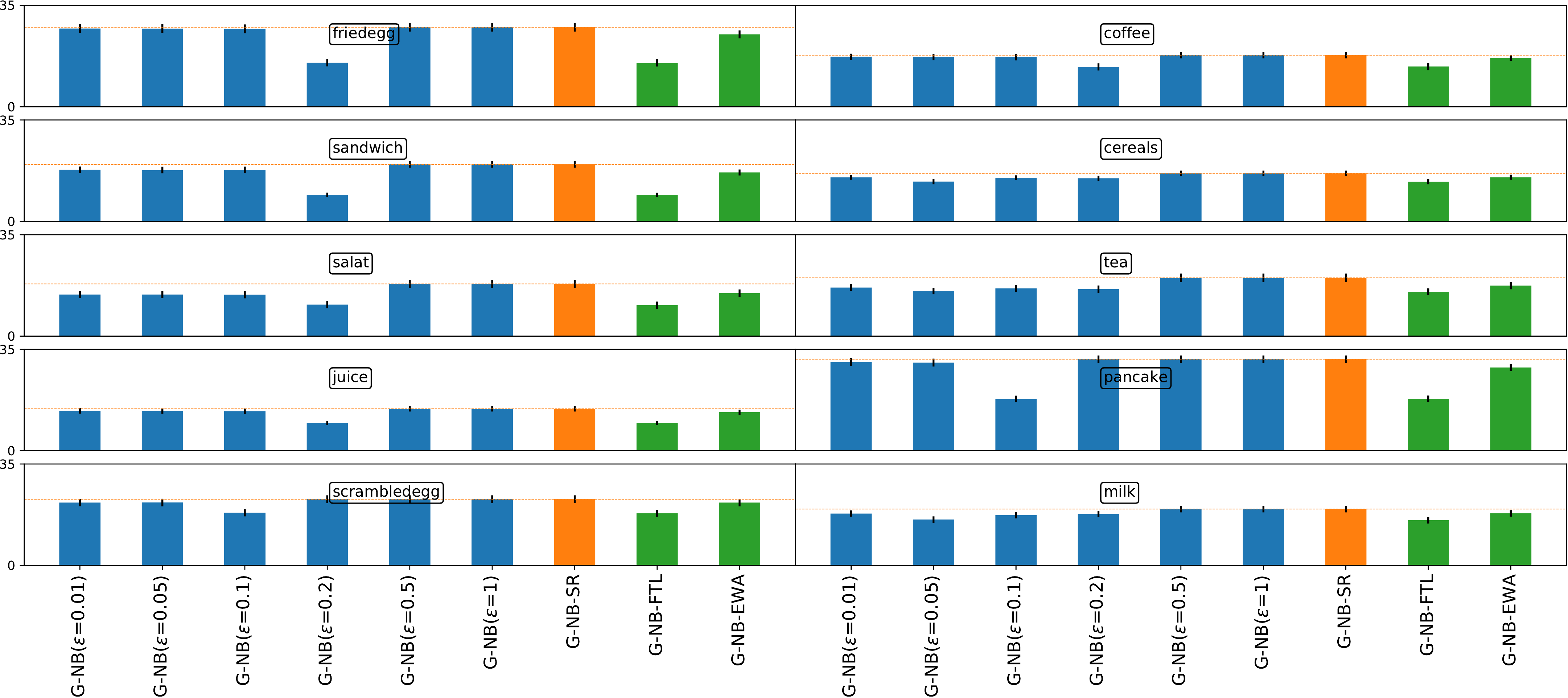}
    \caption{
Performance of MACRO with different subroutines on the Breakfast Actions dataset with the label-based distance function.
%
Each plot corresponds to different action. The y-axes shows error-rates in percents averaged over the persons performing each action. 
The x-axes is labeled by the short name of a subroutine and a threshold used in MACRO.
G-NB-FTL and G-NB-EWA represent the online strategies to choose the threshold.
G-NB-SR is the marginal versions of G-NB algorithm and acts as a baseline.
}
    \label{fig:breakfast-label}
\end{figure*}

\begin{table*}[t]
\scriptsize
\small
\centering
\setlength\tabcolsep{3pt}
\begin{tabular}{l|c|c|c|c|c|c|c}
& \rot{ATL} & \rot{ORD} & \rot{DEN} & \rot{PHX} & \rot{IAH} & \rot{LAS} & \rot{DFW} \\
\hline
ERM-SR & 23.7 & 25.3 & 23.2 & 22.3 & 18.8 & 23.1 & 22.8\\ \hline \hline
ERM-FTL & 22.4 & 23.7 & 22.6 & 22.2 & 17.6 & 23.0 & 21.3\\ \hline
ERM-EWA & 23.1 & 25.0 & 23.5 & 22.4 & 18.3 & 23.2 & 21.9\\ \hline \hline
ERM($\epsilon\!=\!0.15$)& 24.0 & 27.0 & 24.8 & 22.6 & 19.0 & 23.5 & 22.7\\ \hline
ERM($\epsilon\!=\!0.17$)& 23.9 & 26.4 & 24.3 & 22.5 & 18.9 & 23.6 & 22.4\\ \hline
ERM($\epsilon\!=\!0.19$)& 23.5 & 26.0 & 24.5 & 22.5 & 18.9 & 23.5 & 22.4\\ \hline
ERM($\epsilon\!=\!0.22$)& 23.2 & 25.9 & 23.8 & 22.5 & 18.9 & 23.5 & 21.9\\ \hline
ERM($\epsilon\!=\!0.25$)& 23.1 & 25.1 & 23.4 & 22.4 & 18.6 & 23.3 & 21.6\\ \hline
ERM($\epsilon\!=\!0.28$)& 23.0 & 24.3 & 23.1 & 22.4 & 18.2 & 23.2 & 21.4\\ \hline
ERM($\epsilon\!=\!0.31$)& 22.6 & 24.0 & 22.8 & 22.3 & 17.7 & 23.0 & 21.4\\ \hline
ERM($\epsilon\!=\!0.34$)& 22.8 & 23.7 & 22.6 & 22.2 & 17.6 & 22.9 & 21.3\\ \hline
ERM($\epsilon\!=\!0.37$)& 22.7 & 24.0 & 22.6 & 22.3 & 17.8 & 22.9 & 21.9\\ \hline
ERM($\epsilon\!=\!0.4$)& 22.4 & 24.3 & 22.6 & 22.2 & 17.8 & 22.9 & 22.0\\ \hline
ERM($\epsilon\!=\!0.45$)& 22.3 & 24.5 & 23.2 & 22.2 & 17.8 & 23.0 & 22.1\\ \hline
ERM($\epsilon\!=\!0.47$)& 22.3 & 23.9 & 23.2 & 22.2 & 17.9 & 23.1 & 22.1\\ 
\end{tabular}
\quad
\begin{tabular}{l|c|c|c|c|c|c|c}
& \rot{ATL} & \rot{ORD} & \rot{DEN} & \rot{PHX} & \rot{IAH} & \rot{LAS} & \rot{DFW} \\
\hline
VW-SR & 25.1 & 26.9 & 23.8 & 22.7 & 20.2 & 23.5 & 24.8\\ \hline \hline
VW-FTL & 22.5 & 24.0 & 22.8 & 22.3 & 17.8 & 23.1 & 21.6\\ \hline
VW-EWA & 23.1 & 25.3 & 23.9 & 22.8 & 18.5 & 23.5 & 22.7\\ \hline \hline
VW($\epsilon\!=\!0.15$)& 24.4 & 27.6 & 26.2 & 24.4 & 19.3 & 24.7 & 23.2\\ \hline
VW($\epsilon\!=\!0.17$)& 23.7 & 26.6 & 25.3 & 24.1 & 18.9 & 24.5 & 22.9\\ \hline
VW($\epsilon\!=\!0.19$)& 23.3 & 25.8 & 24.8 & 23.6 & 18.8 & 24.5 & 22.5\\ \hline
VW($\epsilon\!=\!0.22$)& 23.2 & 25.6 & 23.9 & 23.0 & 18.8 & 23.8 & 22.1\\ \hline
VW($\epsilon\!=\!0.25$)& 23.2 & 25.0 & 23.3 & 22.6 & 18.6 & 23.5 & 21.8\\ \hline
VW($\epsilon\!=\!0.28$)& 23.1 & 24.7 & 23.0 & 22.3 & 18.4 & 23.2 & 21.6\\ \hline
VW($\epsilon\!=\!0.31$)& 22.9 & 24.4 & 23.0 & 22.2 & 18.1 & 23.2 & 22.0\\ \hline
VW($\epsilon\!=\!0.34$)& 23.1 & 24.0 & 22.8 & 22.3 & 17.8 & 23.1 & 21.9\\ \hline
VW($\epsilon\!=\!0.37$)& 23.1 & 24.4 & 22.7 & 22.2 & 18.1 & 23.1 & 23.5\\ \hline
VW($\epsilon\!=\!0.4$)& 22.7 & 25.2 & 22.8 & 22.2 & 18.1 & 23.0 & 23.5\\ \hline
VW($\epsilon\!=\!0.45$)& 22.4 & 25.8 & 23.9 & 22.2 & 18.2 & 23.3 & 23.6\\ \hline
VW($\epsilon\!=\!0.47$)& 22.5 & 24.2 & 23.9 & 22.2 & 18.3 & 23.4 & 23.6\\ 
\end{tabular}
\caption{
Performance of MACRO with different subroutines on DataExpo Airline data with feature-based distance function. 
The columns correspond to different airports label by their IATA code.
The rows are labeled by the short name of a subroutine and a threshold used in MACRO.
ERM-FTL, VW-EWA, ERM-FTL and VW-EWA represent the online strategies to choose the threshold.
ERM-SR and VW-SR are the marginal versions of the subroutine algorithms and act as a baselines.
The numbers are error-rates in percents.
}
\label{fig:airports-bottleneck-nums}
\end{table*}
\begin{table*}[t]
\scriptsize
\small
\centering
\setlength\tabcolsep{3pt}
\begin{tabular}{l|c|c|c|c|c|c|c}
& \rot{ATL} & \rot{ORD} & \rot{DEN} & \rot{PHX} & \rot{IAH} & \rot{LAS} & \rot{DFW} \\
\hline
ERM-SR & 23.7 & 25.3 & 23.2 & 22.3 & 18.8 & 23.1 & 22.8\\ \hline \hline
ERM-FTL & 22.9 & 23.7 & 22.8 & 22.3 & 17.8 & 23.1 & 21.5\\ \hline
ERM-EWA & 22.5 & 24.2 & 22.9 & 22.3 & 17.8 & 23.1 & 21.7\\ \hline \hline
ERM($\epsilon\!=\!0.005$)& 22.4 & 23.8 & 23.0 & 22.5 & 17.7 & 23.4 & 21.1\\ \hline
ERM($\epsilon\!=\!0.01$)& 22.4 & 23.8 & 22.9 & 22.5 & 17.6 & 23.4 & 21.1\\ \hline
ERM($\epsilon\!=\!0.05$)& 22.2 & 23.7 & 22.9 & 22.4 & 17.7 & 23.2 & 21.1\\ \hline
ERM($\epsilon\!=\!0.1$)& 22.4 & 23.7 & 22.9 & 22.3 & 17.8 & 23.1 & 21.3\\ \hline
ERM($\epsilon\!=\!0.2$)& 22.2 & 23.7 & 22.7 & 22.3 & 17.7 & 23.1 & 21.1\\ \hline
ERM($\epsilon\!=\!0.3$)& 22.3 & 23.7 & 22.7 & 22.3 & 17.6 & 22.9 & 21.8\\ \hline
ERM($\epsilon\!=\!0.4$)& 22.4 & 24.4 & 22.7 & 22.2 & 17.6 & 23.0 & 21.7\\ \hline
ERM($\epsilon\!=\!0.5$)& 22.3 & 24.4 & 22.8 & 22.2 & 17.5 & 23.0 & 21.7\\ 
\end{tabular}
\quad
\begin{tabular}{l|c|c|c|c|c|c|c}
& \rot{ATL} & \rot{ORD} & \rot{DEN} & \rot{PHX} & \rot{IAH} & \rot{LAS} & \rot{DFW} \\
\hline
VW-SR & 25.1 & 26.9 & 23.8 & 22.7 & 20.2 & 23.5 & 24.8\\ \hline \hline
VW-FTL & 22.1 & 23.9 & 23.2 & 22.2 & 17.6 & 23.1 & 21.2\\ \hline
VW-EWA & 22.7 & 24.8 & 23.2 & 22.4 & 18.0 & 23.3 & 22.6\\ \hline \hline
VW($\epsilon\!=\!0.005$)& 22.4 & 23.9 & 23.1 & 22.5 & 17.7 & 23.5 & 21.3\\ \hline
VW($\epsilon\!=\!0.01$)& 22.3 & 23.8 & 23.0 & 22.5 & 17.6 & 23.4 & 21.2\\ \hline
VW($\epsilon\!=\!0.05$)& 22.0 & 23.7 & 22.9 & 22.4 & 17.7 & 23.2 & 21.1\\ \hline
VW($\epsilon\!=\!0.1$)& 22.3 & 23.9 & 22.8 & 22.3 & 17.8 & 23.2 & 21.4\\ \hline
VW($\epsilon\!=\!0.2$)& 22.3 & 23.8 & 22.8 & 22.3 & 17.8 & 23.1 & 21.3\\ \hline
VW($\epsilon\!=\!0.3$)& 22.4 & 24.1 & 22.9 & 22.2 & 17.9 & 23.1 & 22.5\\ \hline
VW($\epsilon\!=\!0.4$)& 22.5 & 25.5 & 22.9 & 22.2 & 17.9 & 23.1 & 22.6\\ \hline
VW($\epsilon\!=\!0.5$)& 22.4 & 25.6 & 22.8 & 22.2 & 17.5 & 23.3 & 22.9\\ 
\end{tabular}
\caption{
Performance of MACRO with different subroutines on DataExpo Airline data with label-based distance function. 
The columns correspond to different airports label by their IATA code.
The rows are labeled by the short name of a subroutine and a threshold used in MACRO.
ERM-FTL, VW-EWA, ERM-FTL and VW-EWA represent the online strategies to choose the threshold.
ERM-SR and VW-SR are the marginal versions of the subroutine algorithms and act as a baselines.
The numbers are error-rates in percents.
}
\label{fig:airports-label-nums}
\end{table*}
\begin{table*}[t]
\scriptsize
\small
\centering
\setlength\tabcolsep{3pt}
\begin{tabular}{l|c|c|c|c|c|c|c|c|c|c}
& \rot{friedegg} & \rot{coffee} & \rot{sandwich} & \rot{cereals} & \rot{salat} & \rot{tea} & \rot{juice} & \rot{pancake} & \rot{scrambledegg} & \rot{milk} \\
\hline
G-NB-SR & 27.5 & 17.8 & 19.7 & 16.6 & 18.0 & 20.1 & 14.5 & 31.6 & 22.9 & 19.4\\ \hline \hline
G-NB-FTL & 8.6 & 8.9 & 7.2 & 7.2 & 8.5 & 8.6 & 5.4 & 8.0 & 6.3 & 8.3\\ \hline
G-NB-EWA & 13.5 & 12.6 & 11.7 & 11.4 & 11.2 & 13.0 & 9.0 & 13.2 & 11.1 & 11.8\\ \hline \hline
G-NB($\epsilon\!=\!1$)& 21.6 & 16.2 & 18.0 & 16.0 & 14.6 & 18.2 & 13.7 & 23.7 & 20.4 & 16.2\\ \hline
G-NB($\epsilon\!=\!5$)& 20.3 & 15.9 & 17.6 & 15.2 & 13.2 & 16.3 & 12.7 & 20.4 & 18.0 & 15.2\\ \hline
G-NB($\epsilon\!=\!7$)& 18.4 & 15.5 & 15.5 & 14.5 & 12.4 & 14.6 & 12.5 & 18.0 & 15.2 & 14.2\\ \hline
G-NB($\epsilon\!=\!10$)& 13.3 & 12.7 & 11.7 & 12.1 & 10.9 & 13.6 & 10.3 & 12.1 & 10.6 & 11.2\\ \hline
G-NB($\epsilon\!=\!15$)& 9.0 & 10.4 & 7.9 & 7.9 & 9.8 & 10.2 & 6.8 & 8.1 & 6.8 & 8.8\\ \hline
G-NB($\epsilon\!=\!17$)& 9.5 & 10.1 & 8.3 & 8.3 & 9.8 & 10.4 & 5.6 & 9.6 & 7.3 & 9.3\\ \hline
G-NB($\epsilon\!=\!20$)& 12.4 & 12.8 & 10.5 & 11.1 & 11.6 & 11.9 & 6.5 & 13.2 & 9.9 & 11.7\\ \hline
G-NB($\epsilon\!=\!25$)& 18.2 & 15.3 & 13.9 & 13.8 & 14.6 & 16.5 & 10.3 & 21.4 & 14.7 & 15.4\\ \hline
G-NB($\epsilon\!=\!30$)& 23.0 & 16.6 & 16.7 & 15.3 & 16.3 & 18.8 & 12.9 & 27.1 & 18.5 & 17.5\\ \hline
G-NB($\epsilon\!=\!35$)& 26.0 & 17.4 & 18.8 & 16.0 & 17.2 & 19.5 & 14.0 & 29.2 & 21.3 & 18.9\\ \hline
G-NB($\epsilon\!=\!40$)& 26.6 & 17.6 & 19.4 & 16.4 & 17.7 & 20.1 & 14.4 & 30.7 & 22.3 & 19.2\\ 
\end{tabular}
\caption{
Performance of MACRO with different subroutines on Breakfast Actions data with feature-based distance function. 
The columns correspond to different airports label by their IATA code.
The rows are labeled by the short name of a subroutine and a threshold used in MACRO.
G-NB-FTL and G-NB-EWA represent the online strategies to choose the threshold.
G-NB-SR is the marginal versions of G-NB algorithm and acts as a baseline.
The numbers are error-rates in percents.
}
\label{fig:breakfast-bottleneck-nums}
\end{table*}
\begin{table*}[t]
\scriptsize
\small
\centering
\setlength\tabcolsep{3pt}
\begin{tabular}{l|c|c|c|c|c|c|c|c|c|c}
& \rot{friedegg} & \rot{coffee} & \rot{sandwich} & \rot{cereals} & \rot{pancake} & \rot{tea} & \rot{juice} & \rot{salat} & \rot{scrambledegg} & \rot{milk} \\
\hline
G-NB-SR & 27.5 & 17.8 & 19.7 & 16.6 & 31.6 & 20.1 & 14.5 & 18.0 & 22.9 & 19.4\\ \hline \hline
G-NB-FTL & 15.2 & 13.9 & 9.2 & 13.7 & 17.9 & 15.3 & 9.6 & 10.7 & 18.0 & 15.6\\ \hline
G-NB-EWA & 25.0 & 16.8 & 16.9 & 15.3 & 28.7 & 17.4 & 13.3 & 14.8 & 21.7 & 18.0\\ \hline \hline
G-NB($\epsilon\!=\!0.01$)& 27.0 & 17.3 & 17.8 & 15.2 & 30.6 & 16.8 & 13.8 & 14.3 & 21.7 & 17.9\\ \hline
G-NB($\epsilon\!=\!0.05$)& 27.0 & 17.2 & 17.7 & 13.7 & 30.4 & 15.5 & 13.7 & 14.3 & 21.7 & 15.8\\ \hline
G-NB($\epsilon\!=\!0.1$)& 26.9 & 17.1 & 17.8 & 15.1 & 17.9 & 16.4 & 13.7 & 14.3 & 18.2 & 17.4\\ \hline
G-NB($\epsilon\!=\!0.2$)& 15.2 & 13.8 & 9.2 & 14.9 & 31.6 & 16.2 & 9.6 & 10.9 & 22.9 & 17.7\\ \hline
G-NB($\epsilon\!=\!0.5$)& 27.5 & 17.8 & 19.7 & 16.6 & 31.6 & 20.1 & 14.5 & 18.0 & 22.9 & 19.4\\ \hline
G-NB($\epsilon\!=\!1$)& 27.5 & 17.8 & 19.7 & 16.6 & 31.6 & 20.1 & 14.5 & 18.0 & 22.9 & 19.4\\ 
\end{tabular}
\caption{
Performance of MACRO with different subroutines on Breakfast Actions data with label-based distance function. 
The columns correspond to different airports label by their IATA code.
The rows are labeled by the short name of a subroutine and a threshold used in MACRO.
G-NB-FTL and G-NB-EWA represent the online strategies to choose the threshold.
G-NB-SR is the marginal versions of G-NB algorithm and acts as a baseline.
The numbers are error-rates in percents.
}
\label{fig:breakfast-label-nums}
\end{table*}

\end{document}